%% file: main.tex
\documentclass[runningheads]{llncs}
\usepackage[utf8]{inputenc}
\usepackage{graphicx}
\usepackage{url}

\usepackage{amsmath}

\usepackage{amsmath}

\usepackage{amsthm}

\usepackage{amsfonts}
\usepackage[boxed,linesnumbered]{algorithm2e}

\newtheorem{method}{Method}
\newtheorem{lemmaX}[theorem]{Lemma}
\newcommand{\figaddr}[1]{#1}
\newcommand{\Bem}[1]{}

\usepackage[misc,geometry]{ifsym}

\begin{document}

\title{How To Overcome Richness Axiom Fallacy}

\author{  Mieczys{\l}aw A. K{\l}opotek \Letter\inst{1}\orcidID{0000-0003-4685-7045}  
      \and \\
        Robert A. K{\l}opotek\inst{2}\orcidID{0000-0001-9783-4914}   
       }
\authorrunning{  M.A. K{\l}opotek  and  R.A. K{\l}opotek}
%
\institute{Institute of Computer Science, \\ Polish Academy of Sciences,
Warsaw, Poland \\
\email{mieczyslaw.klopotek@ipipan.waw.pl}\\
  \and
Faculty of Mathematics and Natural Sciences. School of Exact Sciences,\\
Cardinal Stefan  Wyszy{\'n}ski   University in Warsaw, Poland\\
\email{r.klopotek@uksw.edu.pl}
}

\maketitle

\begin{abstract}
The paper points at the grieving problems implied by the richness axiom in the Kleinberg's axiomatic system and suggests resolutions. 
The richness induces learnability problem in general and leads to conflicts with consistency axiom.  As a resolution, learnability constraints and usage of centric consistency or restriction of the domain of considered clusterings to super-ball-clusterings is proposed. 

\end{abstract}

\section{Introduction}\label{sec:intro}

This is an extended version of ISMIS 2022 Conference paper  \cite{ISMIS:2022:richness}. 

Kleinberg \cite{Kleinberg:2002} introduced an axiomatic system for distance based clustering functions consisting of three axioms: \emph{scale-invariance} (same clustering should be obtained if all distances between objects are multiplied by the same positive number), \emph{consistency} (same clustering should be obtained, if distances within a cluster are decreased and distances between elements from distinct clusters are increased) and \emph{richness} (for any partition of the {data}set into non-empty subsets there should exist a set of distances between the {data}points so that the clustering function delivers this partition). 
We shall say that (1)
 decreasing of distances within a cluster and increasing distances between elements from distinct clusters is called 
\emph{consistency transformation}, and (2)  multiplying 
distances between objects by the same positive number  is called 
\emph{scale-invariance transformation}.

This set of axioms was controversial from the very beginning. Kleinberg proved himself that the entire axiomatic system is contradictory, and only pairs of these axioms are not contradictory. But also the axiom of consistency turned out to be controversial as it excludes the $k$-means algorithm from the family of clustering functions. It has been shown in \cite{MAKRAK:2020:fixdimcons}  that consistency is contradictory by itself if we consider fixed dimensional Euclidean spaces which are natural domain of $k$-means application. 

It seems to be disastrous for the domain of clustering algorithms if an axiomatic system consisting of ''natural axioms'' is self-contradictory. It renders the entire domain questionable. Therefore numerous efforts have been made to cure such a situation by proposing different axiom sets or modifying Kleinberg's one,  just to mention \cite{%
Ackerman:2010NIPS,%
Ackerman:2013,%
Ben-David:2005,%
Ben-David:2009,%
Hopcroft:2012,%
vanLaarhoven:2014,%
Meila:2005,%
Strazzeri:2021,%
Zadeh:2009%
} 
etc.  
Kleinberg himself introduced the concept of partition $\Gamma'$ being a refinement of a partition $\Gamma$, 
if for every set $C' \in \Gamma'$, there is a
set $C \in  \Gamma$ such that $C' \subseteq C$.
He  defines Refinement-Consistency, a relaxation of Consistency, to require that
if distance d' is an f (d)-transformation of d, then f(d') should be a refinement of f(d).
Though  there is no clustering function that satisfies Scale-{In}variance, Richness, and Refinement-Consistency, but if one defines Near-Richness as Richness without  the partition in which each element is in a separate cluster, then  there exist clustering functions f that satisfy Scale-{In}variance and
Refinement-Consistency, and Near-Richness (e.g.  single-linkage with the distance-($\alpha\delta$) stopping condition, where
$\delta=  min_{i,j} d(i, j)$ and $\alpha\ge 1$.  

In this paper, we look at more detail at the richness property and investigate its counter-intuitiveness. 
The richness induces learnability problem in general (see Sec.\ref{sec:learnability}) and leads to conflicts with consistency axiom  (see Sec.\ref{sec:consistency}). In the past, we applied embedding into Euclidean space to resolve some problems with consistency axiom, but it does not work for richness  (Sec.\ref{sec:euclidean}). Therefore, as a resolution, 
usage of centric consistency (Sec.\ref{sec:consistency}) or restriction of the domain of considered clusterings to super-ball-clusterings (Sec.\ref{sec:superball}) is proposed. 
\section{Richness and {Learn}ability} \label{sec:learnability}
 
 Following \cite{Kleinberg:2002} let us define: 
A \emph{partition} $\Gamma$ of the set of {data}points $S$ into $k$ partitions is the set $\Gamma=\{C_1,\dots,C_k\}$ such that $C_i\cap C_j=\emptyset $ for $i\ne j$, $C_i\ne \emptyset$ and $S=C_1\cup C_2\cup \dots \cup C_k$.
A \emph{clustering function} $f$ is a function assigning a partition $\Gamma$ to any  dataset $S$ with at least two {data}points.
That is given a clustering function operates on the dataset $X$ from which samples $S$ are taken, then  $f: 2^X\rightarrow 2^{2^X}$, where for any $S\subset X$, $card(S)\ge 2$, $f(S)$ is a partition of $S$.
A \emph{clustering quality function} $Q$ is a function assigning a real value to a partition. That is $Q: 2^{2^X}\rightarrow \mathbb{R}$, where $Q(\Gamma)$ is defined only if $Q$ is a partition of some set. 
 
Richness is a concept from the realm of clustering, that is unsupervised learning. Nonetheless, if we have an explicit clustering quality function $Q$, but  the  clustering function $f$ has to be deduced from it assuming that $Q(f(S))$ is optimal among all possible partitions of $S$, then we have to do with the typical supervised learning task. 
Hence we are transferred to the supervised learning domain. 
In this domain, a function is deemed {learn}able if it can be discovered in polynomial time of problem parameters.

\begin{theorem}
There exist rich clustering functions that cannot be learned in polynomial time.
\end{theorem}
\begin{proof}
The proof is by construction of such a function. 
Assume we have a set $S$ of cardinality $n$ with a distance function $d$ (assume "distances" can be defined as Kleinberg assumes, being greater than zero for distinct points, symmetric  and arbitrary otherwise). 
Let $G$ be the set of all possible partitions of $S$. Let $m: G\rightarrow \mathbb{N}\cup \{0\}$ be a function assigning each element of $G$ a unique consecutive integer starting from 0. 
This mapping is created as follows.
Let the  number of all possible partitions of $S$ into nonempty sets be $M+1$.
Assume $P,R$ is a pair of points with the maximal distance within $d$.  For each possible clustering $C_j$, $j=0,...,M$ (possible partition into nonempty sets)  we compute the average distance of {data}points within clusters, but neglecting the distance between the two points P,R, if they happen to fall into the same cluster. We sort $C_1,\dots C_m$  on them assigning 0 to the 
cluster with the lowest average distance and so on. On ties, for the clusterings affected, we try averages of squares of distances an so on. If no resolution possible, arbitrary decision is made, based e.g. on alphabetic ordering of clusters, cluster sizes etc.  
Create a function $h$ assigning each integer $i\in  [0,M]$ an interval $(i/M,(i+1)/M]$  
Imagine the following \textit{clustering function} $f$.
It computes the quotient $q$ of the smallest distance between data points to the highest one, that is between P and R. 
Then the clustering function $f(S,d)=m^{-1}(h^{-1}(q))$.
The richness of $f$ is easy to achieve. In case of equidistant points in a set $S$, $q$ amounts to 1. By increasing the distance $d(P,Q)$ alone (permitted by Kleinberg's definition), $q$ can take any value  from  (0,1] interval. So there exists always a set $S$ and distance function $d$ such that $f$ returns the desired clustering. 
So the clustering function has the richness property. Because of relying on quotients $q$, it is also scale-invariant. 
Though the function is simple in principle (and useless also), and meets axioms of richness and scale -invariance, we have a practical problem:
As no other limitations are imposed, one has to check 
up to 
$ M+1= \sum_{k=2}^n
\frac{1}{k!}\sum_{j=1}^{k}(-1)^{k-j}\Big(\begin{array}{l}k\\j\end{array}\Big)j^n \label{CLU:eq-liczebnosc} 
$
possible partitions (Bell number) in order to verify which one of them is the best for a given distance function 
because there must exist at least one distance function suitable for each of them. 
This 
cannot be done in reasonable time even if computation of $q$  is polynomial (even linear) in the dimensions of the task ($n$). 
\end{proof}

The idea of learnability is based on the assumption that the learned concept may be verified against new data. For example, if one learned $k$-means clustering on a data sample, then with newly incoming elements the clustering should not change (too much). 
With the richness requirement, if you learnt the model from a sample of size $n$, then it consists of no more than $n$ clusters. But there exists a sample of size $2n$ such that it defies the learned structure, because it has more than $n$ clusters.  That is the clustering structure of the domain  is not {learn}able by definition of richness.

Furthermore, most algorithms of cluster analysis are constructed in an incremental way.
But this can be useless if the clustering quality function is designed in a very unfriendly way, for example based on modulo computation. 
\begin{theorem}
There exist rich clustering functions that cannot be learnt incrementally
\end{theorem}
\begin{proof}
The proof is by construction of such a function. 
Consider the following clustering quality function $Q_s(\Gamma)$. 
For each data{}point $i\in S$ we compute its distance $d_i$ to its cluster center under $\Gamma$. Let $dmx= \max_{i\in S }d_i$. $q=\sum_{i\in S} round(10000d_i/dmx)$. Then $Q_s(\Gamma)=p - q \mod p$, where $p$ is a prime number (here 3041). 
This $Q_s$ was applied in order to  partition first $n$ points from sample data from Table \ref{tab:richnessdata}, as illustrated in Table 
\ref{tab:bestpartitionrichnessdata}.
 It turns out that the best partition 
for $n$ points does not give any hint for the best partition for $n+1$ points 
therefore each possible partition needs to be investigated in order to find the best one.%
\footnote{
Strict separation   \cite{Blum:2009} mentioned earlier is another kind of 
a weird cluster quality function, requiring visits to all the partitions 
}

\begin{table}
\centering
\caption{Data points to be clustered using a ridiculous clustering quality function  }
\label{tab:richnessdata}
\begin{tabular}{lll}
\hline 
id & $x$ coordinate  & $y$ coordinate   \\
\hline 
 1& 4.022346 &5.142886	\\
  2& 3.745942& 4.646777	\\
  3& 4.442992& 5.164956	\\
  4& 3.616975& 5.188107	\\
  5& 3.807503& 5.010183	\\
  6& 4.169602& 4.874328	\\
  7& 3.557578& 5.248182	\\
  8& 3.876208& 4.507264	\\
  9& 4.102748& 5.073515	\\
10& 3.895329& 4.878176\\ 
\hline 
\end{tabular}
\end{table}

\begin{table}
\centering
\caption{Partition of the best quality 
(the lower the value the better)
after including $n$ first points from Table \ref{tab:richnessdata}.   }
\label{tab:bestpartitionrichnessdata}
\begin{tabular}{llp{5.5cm}}
\hline
$n$ & quality & partition \\
\hline 
  2&1270  & \{ 1, 2 \}\\                                  
  3&1270  & \{ 1, 2 \}  \{ 3 \}\\                           
  4&823   &\{ 1, 3, 4 \}  \{ 2 \}\\                         
  5&315   &\{ 1, 4 \}  \{ 2, 3, 5 \}\\                      
  6&13   &\{ 1, 5 \}  \{ 2, 4, 6 \}  \{ 3 \}\\                
  7&3   &\{ 1, 6 \}  \{ 2, 7 \}  \{ 3, 5 \}  \{ 4 \}\\          
  8&2   &\{ 1, 2, 4, 5, 6, 8 \}  \{ 3 \}  \{ 7 \}\\           
  9&1   &\{ 1, 2, 4, 5 \}  \{ 3, 8 \}  \{ 6, 9 \}  \{ 7 \}\\    
10&1   &\{ 1, 2, 3, 5, 9 \}  \{ 4, 6 \}  \{ 7, 10 \}  \{ 8 \}\\
\end{tabular}
\end{table}

  \end{proof}

Under these circumstances let us point at the implications of the so-called  
 learnability theory \cite{Valiant:1984},\cite{Fulop:2013} for the richness axiom. 
On the one hand the hypothesis space is too big for learning 
a clustering from a sample. 
On the other hand an exhaustive search in this space is prohibitive so that some theoretical clustering functions do not make practical sense.

Furthermore, if the clustering function can fit any data, we are practically unable to learn any structure of data space from data \cite{MAK0:1991}. And this learning capability is necessary at least in the cases: either when  
the data may be only representatives of a larger population 
or the distances are measured with some measurement error (either systematic or random) or both. 
We speak about a much broader aspect than 
cluster stability or cluster validity,   pointed at by Luxburg \cite{Luxburg:2011,Luxburg:2009}. 

In the special case of $k$-means, the reliable estimation of cluster center position and of the variance in the cluster plays a significant role. 
But there is no reliability for cluster center if a cluster consists of fewer than 2 elements, and for variance the minimal cardinality is 3. 

 
\section{Richness Problems in Euclidean Space}  \label{sec:euclidean}
%
We have shown \cite{MAKRAK:2020:fixdimcons} that some problems with consistency axiom can be resolved by embedding into fixed-dimensional space\footnote{For example, $k$-means has the property of consistency in one dimensional space, though it does not have this property in general. }. But we show here that this is not the case for richness.  
\begin{figure}
\centering
(a)\includegraphics[width=0.45\textwidth]{\figaddr{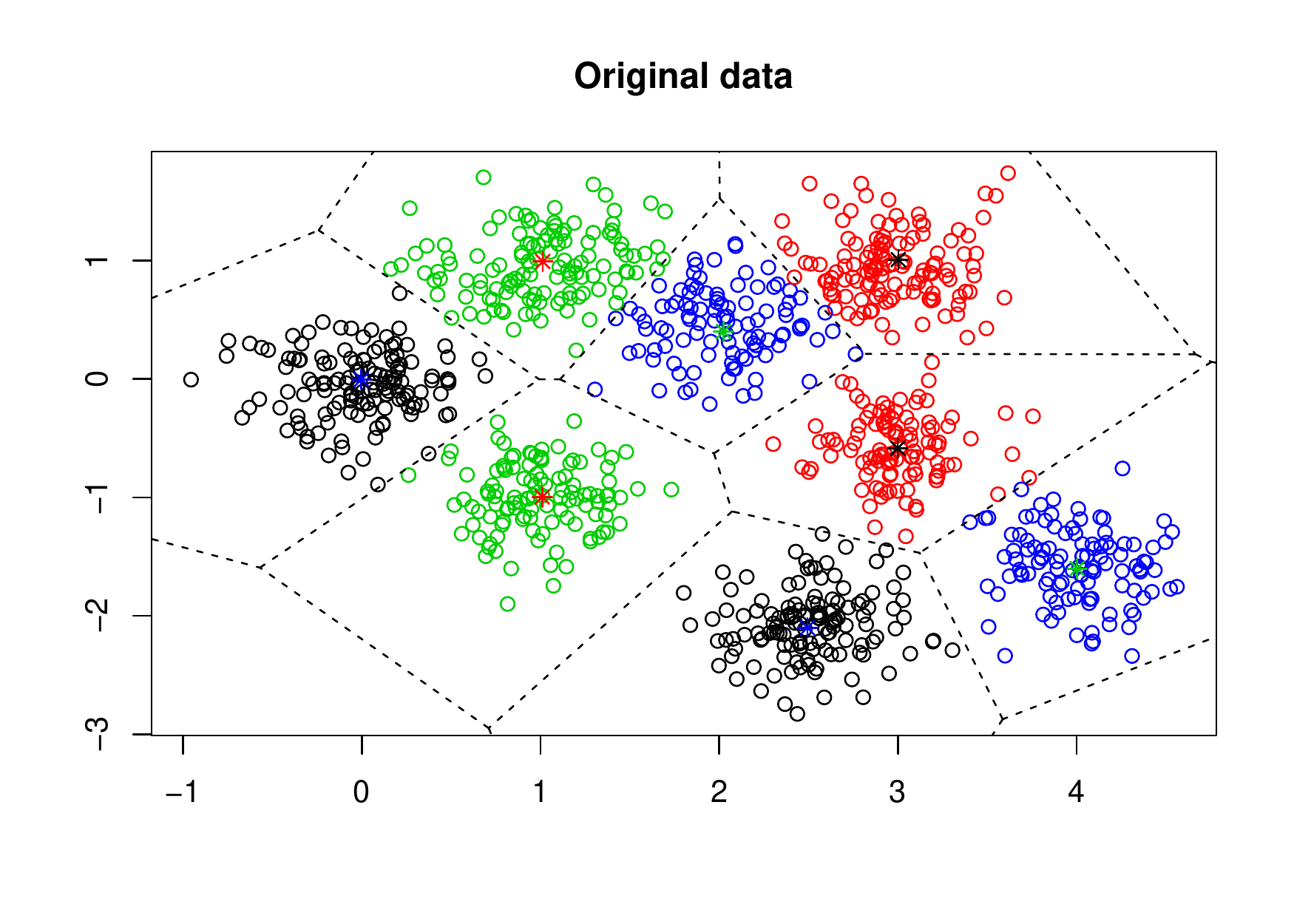}}  %
(b)\includegraphics[width=0.45\textwidth]{\figaddr{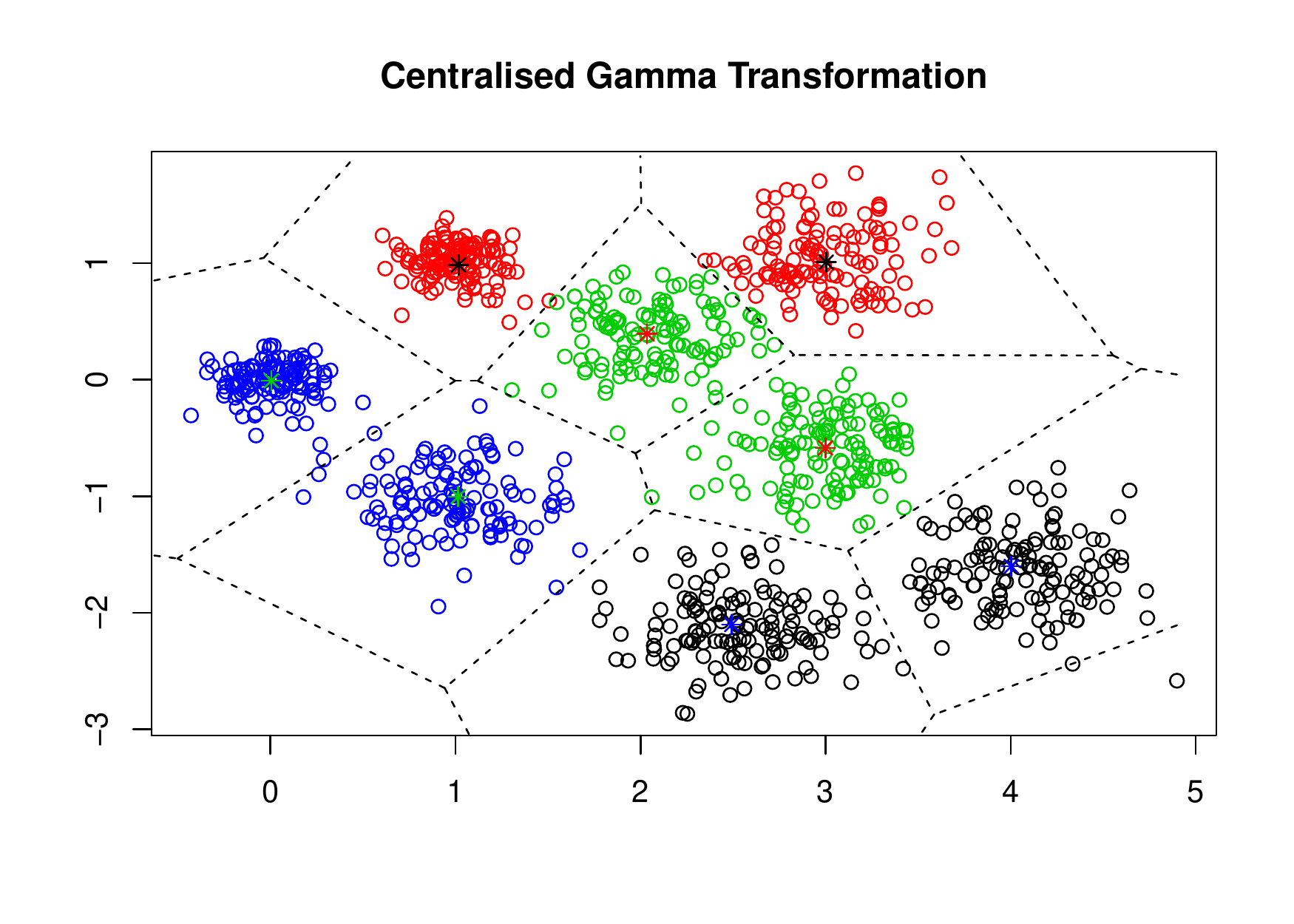}}  %
\caption{(a) A mixture of 8 normal distributions as clustered by $k$-means algorithm (Voronoi diagram superimposed).  
(b) Same data 
after a centric consistency transformation,
 clustered by $k$-means algorithm  into 8 groups
}\label{fig:8clusters}
\label{fig:8clustersKlopotek}
\end{figure}

While consistency transform turns out to be too restrictive in finite dimensional space \cite{MAKRAK:2020:fixdimcons}, the richness is problematic the other way, mainly from the point of view of learnability theory, but also leads to the chaining effect \cite{Kleinberg:2002}.  
 
\Bem{Richness or near-richness forces the introduction of \textit{refinement-consistency} which is a too weak concept. 
But even if we allow for such a resolution of the contradiction in Kleinberg's framework, it still does not make it suitable for practical purposes. 
The most serious drawback of Kleinberg's axioms is the richness requirement.}

Let us ask whether or not it is possible 
to have richness, that is for any partition there exists always a distance function that the clustering function will return this partition, and yet
if we restrict ourselves to $\mathbb{R}^m$, the very same clustering function is not rich any more, or even it is not chaining.

Consider the following clustering function $f()$. 
If it takes a distance function $d()$ that takes on only two distinct values $d_1$ and $d_2$ such that $d_1<0.5 d_2$ and for any three data points $a,b,c$ if $d(a,b)=d_1, d(b,c)=d_1$ then $d(a,c)=d_1$, it creates clusters of points in such a way that $a,b$ belong to the same cluster iff $d(a,b)=d_1$, and otherwise they belong to distinct clusters. If on the other hand $f()$ takes a distance function not exhibiting this property, it works like $k$-means. Obviously, function $f()$ is rich, but at the same time, if confined to $\mathbb{R}^m$, if $n>m+1$ and $k\ll n$, then it is not rich -- it is in fact $k$-rich, and hence not chaining.

Can we get around the problems of  all three Kleinberg's axioms in a similar way in $\mathbb{R}^m$?
Regrettably,

\begin{theorem}{} \label{thm:KleinbergImpossibilityInRm}
If $\Gamma$ is a partition of $n>2$ elements returned by 
a clustering function $f$ under some distance function $d$ (any distance function  allowed by Kleinberg in \cite{Kleinberg:2002}),
and  
 $f$
 satisfies 
consistency, 
then there exists a distance function $d_E$ embedded in 
$\mathbb{R}^m$ for the same set of elements such that 
$\Gamma$ is the partition of this set under $d_E$.
\end{theorem}
Theorem \ref{thm:KleinbergImpossibilityInRm}
implies that the constructs of contradiction of Kleinberg axioms are simply transposed from the domain of any distance functions to distance functions in $\mathbb{R}^m$. 

\begin{proof}
To show the validity of the theorem, we will construct the appropriate distance function 
$d_E$ by embedding in the $\mathbb{R}^m$.
Let $dmax$ be the maximum distance between the considered elements under $d$. 
Let $C_1,\dots,C_k$ be all the clusters contained in $\Gamma$. 
For each cluster $C_i$ we construct a ball $B_i$ in $\mathbb{R}^m$ with radius $r_i$ equal to 
$r_i=\frac12 \min_{x,y\in C_i, x\ne y} d(x,y)$.
The ball $B_1$ will be located in the origin of the coordinate system.
$B_{1,\dots,i}$  be the ball in $\mathbb{R}^m$ containing all the balls $B_1,\dots,B_i$.
Its center be at $c_{1,\dots,i}$ and radius $r_{1,\dots,i}$. 
The ball $B_{i}$ will be located on the surface of the ball 
with center at $c_{1,\dots,i-1}$ and radius $r_{1\dots,i-1}+dmax+r_{i}$.
For each $i=1,\dots,k$, place elements of $C_i$  at distinct locations  within the ball $B_i$. 
Define the distance function $d_E$  as the Euclidean distances within $\mathbb{R}^m$  in these constructed locations. 

$d_E$ is a consistency-transform of $d$, as distances between 
elements of $C_i$ are smaller than  or equal to  $2 r_{i}=\min_{x,y\in C_i, x\ne y} d(x,y)$, 
and the distances between elements of different balls exceed 
$dmax$. 
\end{proof}

This means that if $f$ is rich and consistent, it is also rich in $\mathbb{R}^m$.


\section{Disturbing Effects of Consistency Axiom on Richness Axiom}\label{sec:consistency}

Let us recall practical problems with consistency,
pointed at by  Ben-David \cite{Ben-David:2009}: emergence of impression of a different clustering of data (with a different number of clusters) as they show in their Fig.1.
This may imply problems if we  apply $k$-means algorithm with $k$ varying over a range of values  to identify most appropriate value of $k$, e.g. by  picking  $k$, for which 90\% of variance is variance explained or there is abrupt break (saturation) in increase of relative variance explained.

   \begin{figure}
\centering (a)
\includegraphics[width=0.45\textwidth]{\figaddr{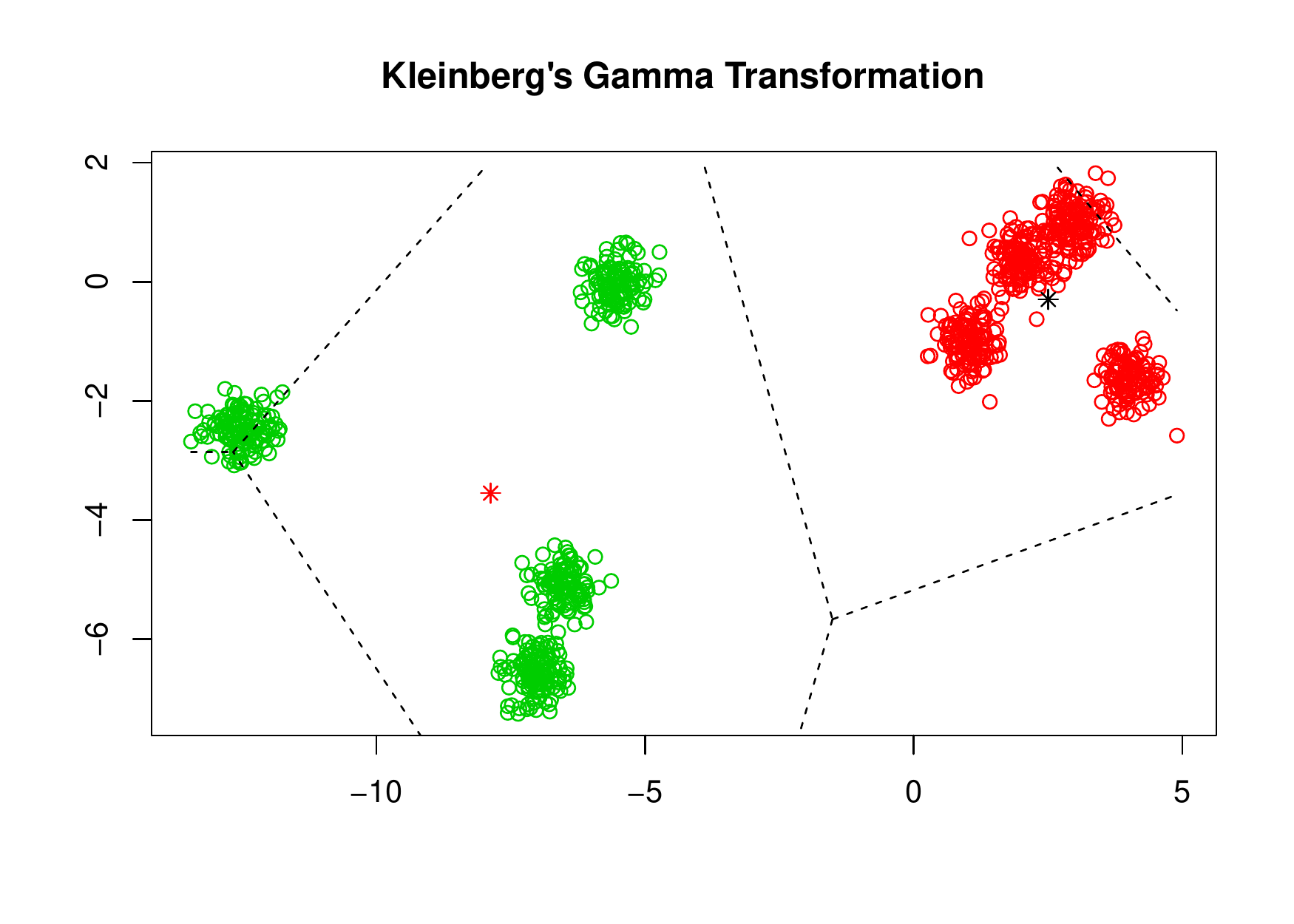}}  %
(b) %
\includegraphics[width=0.45\textwidth]{\figaddr{8clustersKleinbergGamma4gr.pdf}}  %
\caption{Data from 
Fig. \ref{fig:8clusters}(a) after Kleinberg's consistency-transformation 
 clustered by $k$-means algorithm  into (a) two   and  (b) four  groups. 
}\label{fig:8clustersKleinberg}
\end{figure}

The variance explained $Var_E(\Gamma_k)$ by the clustering $\Gamma_k$ of the dataset $S$ by $k$-means algorithm is defined as follows. 
$Var_E(\Gamma_k)=Var(S)-\sum_{c\in \Gamma_k} Var(c)$.
The percentage of variance explained $Var_{RE}(\Gamma_k)=Var_E(\Gamma_k)/Var(S)\cdot 100\%$.

We applied  this method of choosing $k$ for $k$-means (R implementation) to  a mixture of data points sampled from  8 normal distributions, shown in  
  Fig.\ref{fig:8clusters}(a). 
  Let us define: 
  $$RIV(\Gamma_b,\Gamma_n,\Gamma_a)=\frac{ Var_{RE}(\Gamma_{n})
- Var_{RE}(\Gamma_{b,})
}
{ Var_{RE}(\Gamma_{a})
- Var_{RE}(\Gamma_{n})
}$$ and  
$$RIV(k)=RIV(\Gamma_{k-1},\Gamma_k,\Gamma_{k+1})$$
As  visible in Fig.\ref{fig:8clusters}(a) and confirmed by  the   column ``Original'' of Table \ref{tab:comparisonVariance}, representing the percentage of variance explained ($Var_{RE}$), 
 the $k$-means algorithm with $k=8$, separates best the points from various distributions. 
At about 7-8 clusters we get saturation, 90\% explained variance mark is crossed and with more than 8 clusters the relative increase of variance explained $RIV(k)$, having a pick at $k=8$, 
drops abruptly.

Fig. \ref{fig:8clustersKleinberg} illustrates a result of a consistency-transform on the results of the clustering from   Fig. \ref{fig:8clusters}(a) (outer consistency transform).
The ``RIV'' column next to ``Kleinberg'' column picks at $k=2$ and has a smaller pick at $k=4$
Visually  we have two clusters (Fig.\ref{fig:8clustersKleinberg}(a)). 
We could also classify this data set as having four clusters, as in  Fig.   \ref{fig:8clustersKleinberg}(b). 
This renders Kleinberg's consistency axiom counter-intuitive, independently of the choice of the particular clustering algorithm - the $k$-means. 
It  demonstrates also the weakness of outer-consistency concept. %
%
%
%
Kleinberg's consistency transforms may also create new structures, as illustrated in Fig. \ref{fig:ConsTrans}.

\begin{figure}
\centering (a)
\includegraphics[width=0.4\textwidth]{\figaddr{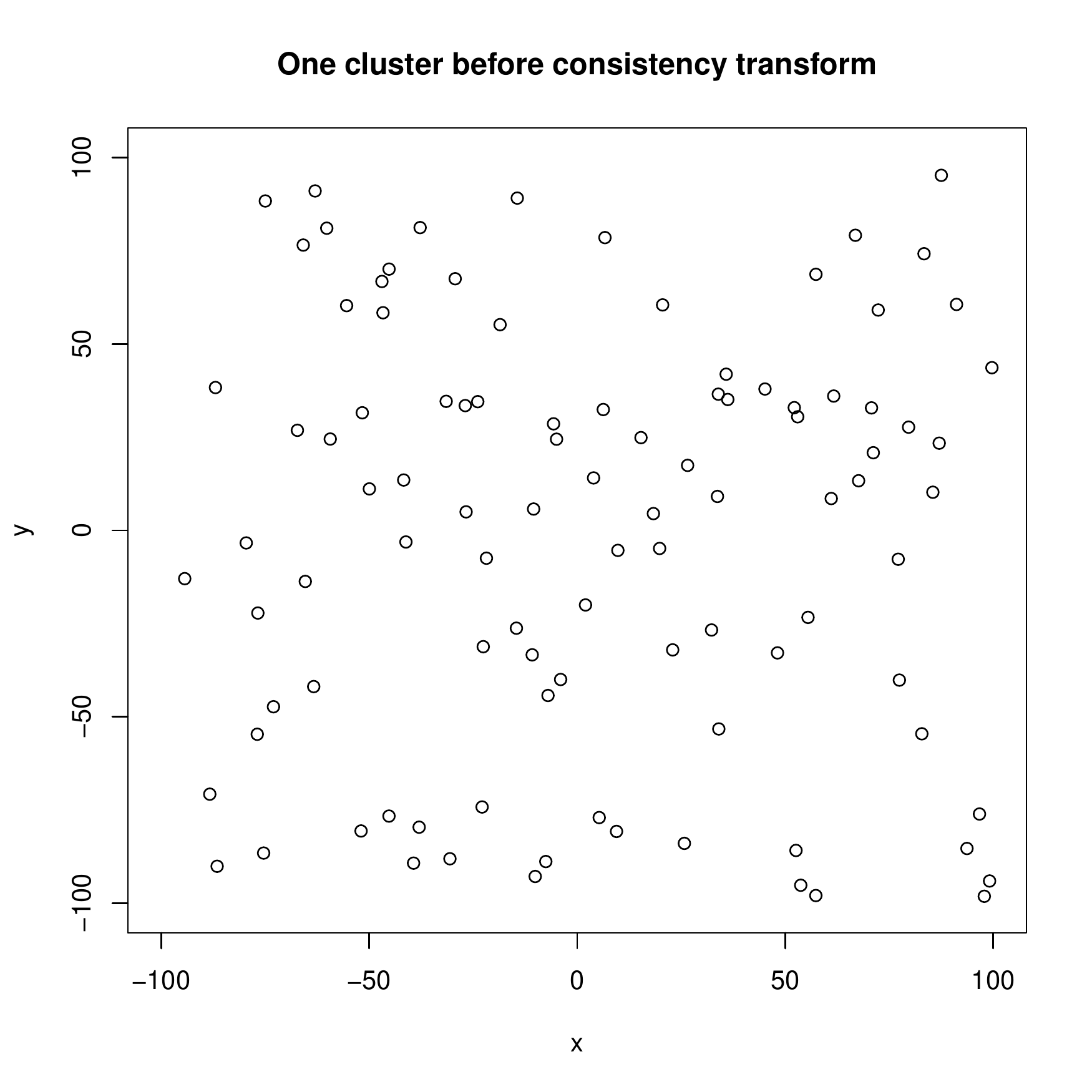}}  %
(b)
\includegraphics[width=0.4\textwidth]{\figaddr{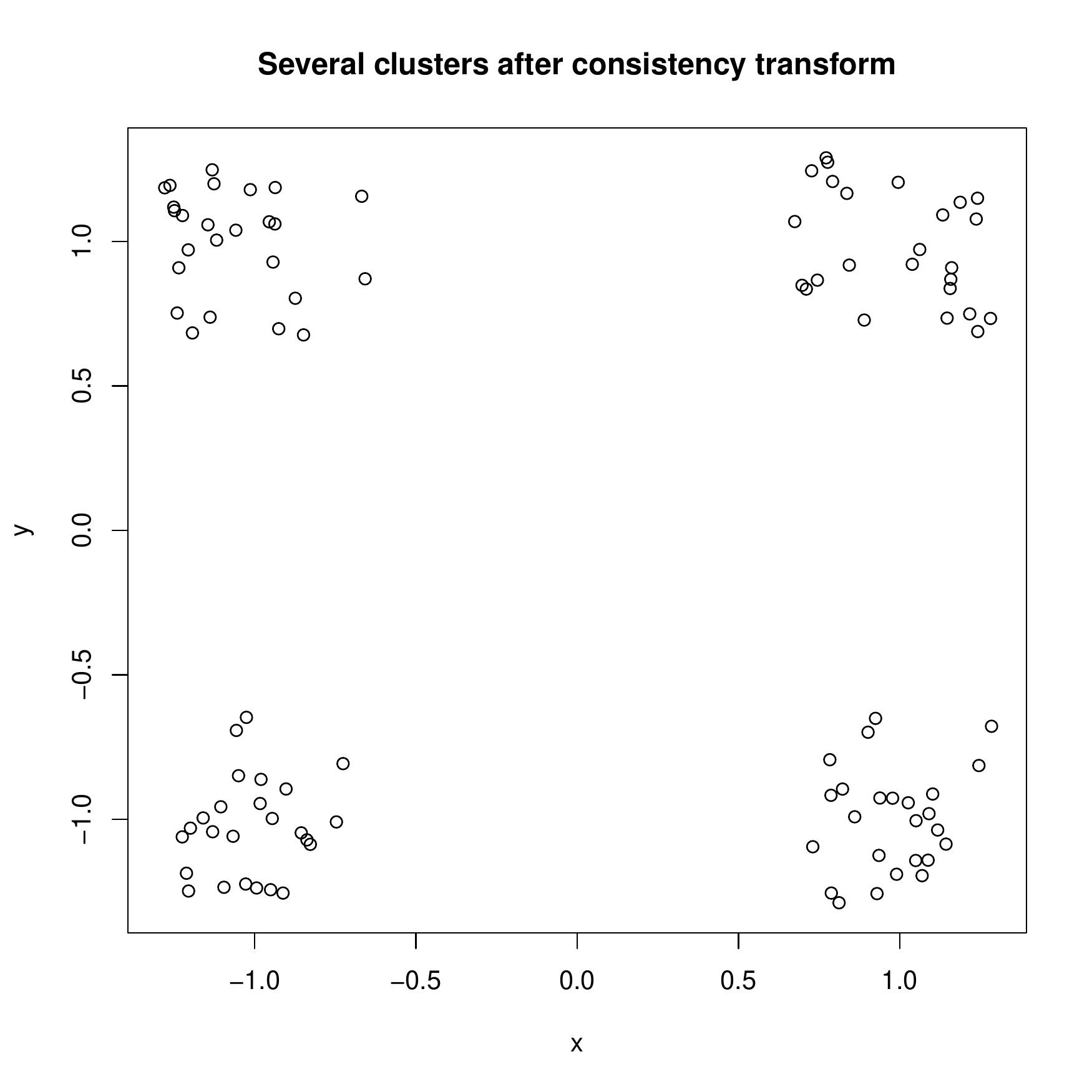}}  %
\caption{Problems with consistency in Euclidean space.  
(a) single cluster before consistency transform. 
(b)  four clusters emerging after consistency transform. 
}\label{fig:ConsTrans}
\end{figure}

The problems with Kleinberg's consistency transformation gave rise to a proposal of centric consistency transformation \cite{Klopotek:2022continuous}. 
The centric consistency transformation
is a linear transformation decreasing the distance of each cluster element  to the gravity center of the cluster by some factor $0<\lambda\le 1$.

\input varTable1

\begin{theorem}
If the RIV  quotient is used for selection of $k$, then centric consistency preserves or strengthens RIV value for the pick $k$. 
\end{theorem}
\begin{proof}
The centric consistency impacts the clustering by $k$-means as follows:
Let $\Gamma$ be an original clustering, $c\in \Gamma_k$ be a cluster to which we apply centric consistency with $\lambda$ yielding a clustering $\Gamma_{k.c,\lambda}$.
$$Var_E(\Gamma_{k,c,\lambda})=Var_E(\Gamma_k)-(1-\lambda^2)Var(c)$$ while the variance of the entire set drops from $Var(S))$ to $Var(S)-(1-\lambda^2)Var(c)$. 
The relative variance is then 
$$ Var_{RE}(\Gamma_{k,c,\lambda})=\frac{
Var_E(\Gamma_k)-(1-\lambda^2)Var(c)
}{
Var(S)-(1-\lambda^2)Var(c)
}$$ 
Obviously,  $Var_{RE}(\Gamma_{k,c,\lambda})\ge Var_{RE}(\Gamma_k) $
\Bem{
Consider the quotient 
$$\frac{ Var_{RE}(\Gamma_{k,c,\lambda})}{  Var_{RE}(\Gamma_k) }
=\frac{Var(S))}{Var(S)-(1-\lambda^2)Var(c)} \frac{Var_E(\Gamma_k)-(1-\lambda^2)Var(c)}{Var_E(\Gamma_k)}
$$
}
Consider the quotient 
$$RIV(\Gamma_{k',c,\lambda},\Gamma_{k,c,\lambda},\Gamma_{k",c,\lambda})=\frac{ Var_{RE}(\Gamma_{k,c,\lambda})
- Var_{RE}(\Gamma_{k',c,\lambda})
}
{ Var_{RE}(\Gamma_{k",c,\lambda})
- Var_{RE}(\Gamma_{k,c,\lambda})
} $$ $$
= \frac{Var_E(\Gamma_{k})-(1-\lambda^2)Var(c)
-Var_E(\Gamma_{k'})+(1-\lambda^2)Var(c)
}{
Var_E(\Gamma_{k"})-(1-\lambda^2)Var(c)
-Var_E(\Gamma_{k})+(1-\lambda^2)Var(c)
}
$$
$$= \frac{Var_E(\Gamma_{k}) 
-Var_E(\Gamma_{k'}) 
}{
Var_E(\Gamma_{k"}) 
-Var_E(\Gamma_{k}) 
}
=RIV(\Gamma_{k'},\Gamma_{k},\Gamma_{k"})
$$
in case $k'<k<k"$ and $c\subseteq c' \in \Gamma_{k'}$ and $c\subseteq c" \in \Gamma_{k"}$. 
If, however, $c$ is not a subset of any set in $\Gamma_{k'}$ and/or $\Gamma_{k"}$, then the  variance reduction  in case of k' and/or k" clustering is not that big so that 
$$RIV(\Gamma_{k',c,\lambda},\Gamma_{k,c,\lambda},\Gamma_{k",c,\lambda})=\frac{ Var_{RE}(\Gamma_{k,c,\lambda})
- Var_{RE}(\Gamma_{k',c,\lambda})
}
{ Var_{RE}(\Gamma_{k",c,\lambda})
- Var_{RE}(\Gamma_{k,c,\lambda})
} $$ $$
= \frac{Var_E(\Gamma_{k})-(1-\lambda^2)Var(c)
-Var_E(\Gamma_{k'})+(1-\lambda^2)Var(c)+\epsilon'
}{
Var_E(\Gamma_{k"})-(1-\lambda^2)Var(c)-\epsilon"
-Var_E(\Gamma_{k})+(1-\lambda^2)Var(c)
}
$$
$$= \frac{Var_E(\Gamma_{k}) 
-Var_E(\Gamma_{k'}) +\epsilon'
}{
Var_E(\Gamma_{k"}) 
-Var_E(\Gamma_{k}) -\epsilon"
}
\ge  \frac{Var_E(\Gamma_{k}) 
-Var_E(\Gamma_{k'}) 
}{
Var_E(\Gamma_{k"}) 
-Var_E(\Gamma_{k}) 
}
=RIV(\Gamma_{k'},\Gamma_{k},\Gamma_{k"})
$$
where $\epsilon'\ge 0, \epsilon"\ge 0$.
This holds simply if, in cluster center triangles formed by cluster $c$ and clusters "splitting" it,  the angle at $c$ cluster center does not exceed $90^o$. 
Otherwise consider the case when the entire $c$ cluster was moved to a single point. Then, in some clustering $\Gamma_x$ one cluster contains it. If we spread the cluster again, and at some point the optimal clustering is $\Gamma_y$ such that $c$ is split between clusters, then by reversing the thinking shrinking $c$ must reduce the quality function of $\Gamma_y$ to be loweer than $\Gamma_x$. So we get the result. Consult also \cite{MAK:2022:kmeanspreserving}. 

\end{proof}

Richness usage  in combination with centric consistency may make sense in general, though  counterexamples can be constructed. 
In \cite{Klopotek:2022continuous} we introduced the concept of 
$k\downarrow$-near-richness. A clustering algorithm possesses this property if it can return any clustering except one with the number of  clusters over $k$. 
We have shown in \cite[Theorem 32]{Klopotek:2022continuous} that an axiomatic system consisting of centric consistency, scale-invariance and $k\downarrow$-nearly-richness is not contradictory.

\section{Super-Ball Separation}\label{sec:superball}
%
Consider a special case of dataset to be clustered. 
\begin{definition}
\Bem{Consider data with the following clusters: For each cluster $i$ there exists a radius $r_i$ such that all members of the cluster lie within $r_i>0$ of the cluster gravity center and for some $\epsilon>0$ at least one cluster element lies further than $\epsilon$ from gravity center. Furthermore the distance between the gravity center of cluster $i$ and $j$ amounts to at least $3r_i+_3r_j$. We will call such clusters super-ball-separated and the clustering itself super-ball-clustering. 
}
Consider data with the following clusters: For each cluster $i$ there exists a radius $r_i$ such that all members of the cluster lie within $r_i>0$ of any cluster member and at least two cluster elements have positive distance. Furthermore the distance between any convex hull point of cluster $i$ and $j$ amounts to at least $sr_i+sr_j$. We will call such clusters $s$-super-ball-separated and the clustering itself $s$-super-ball-clustering. 
\end{definition}
The $s$-super-ball-separation may be seen as a generalization of nice separation and perfect separation as well as their ball variants see e.g. \cite{RAK:MAK:2019:perfectball}.

\begin{theorem}
If the dataset $S$ has a 
$s$-super-ball-clustering $\Gamma_1$ into 
$k_1$ clusters, and one cluster $C\in \Gamma_1$  has a  $s$-super-ball-clustering $\Gamma_2$ into $k_2$ clusters, then this dataset $S$ has a 
$s$-super-ball-clustering $\Gamma$ into 
$k_1+k_2$ clusters of the form  $\Gamma=(\Gamma_1-\{C\})\cup \Gamma_2$
\end{theorem}
\begin{proof}
    All radii $r_i$ in $\Gamma_2$ are smaller or equal to the radius $r$ of $C$ under $\Gamma_1$ so that distances from other elements of $\Gamma_1$ to elements of $\Gamma_2$ fulfills the $s$-super-ball-separation condition. 
\end{proof}

\begin{theorem}
If the dataset $S$ has a 
$s$-super-ball-clustering $\Gamma_1$ into 
$k_1$ clusters, as well a    $s$-super-ball-clustering $\Gamma_2$ into $k_2$ clusters, then 
it is impossible that there exist two clusters $C_a,C_b\in \Gamma_1$
and the third cluster  $C_c\in \Gamma_2$ such that 
$C_a\cap C_c\ne \emptyset$, $C_a \backslash  C_c\ne \emptyset$  and $C_b\cap C_c\ne \emptyset$ for $s>1$. 
\end{theorem}
\begin{proof}
Assume the claim is wrong, that is  
 there exist two clusters $C_a,C_b\in \Gamma_1$
and the third cluster  $C_c\in \Gamma_2$ such that 
$C_a\cap C_c\ne \emptyset$ and $C_b\cap C_c\ne \emptyset$ for $s>1$. 
    Let $P_a \in C_a\cap C_c$ and 
    let $P_b \in C_b\cap C_c$ and let $r_a, r_b, r_c$ be the radii of clusters $C_a,C_b,C_c$ resp. 
    Due to $P_a,P_b\in C_c$, $|P_aP_b|<r_c$. 
    Due to $P_a\in C_a,P_b\not\in C_a$, $|P_aP_b|>sr_a$. 
    Due to $P_a\not\in C_b,P_b\in C_b$, $|P_aP_b|>sr_b$. 
    Therefore $r_c>|P_aP_b|>sr_a$. 
    Let $Q_a$ be such that 
      $Q_a\not\in C_c,Q_a\in C_a$, hence $|Q_aP_a|<r_a$
      and at the same time  $|Q_aP_a|>sr_c$. Therefore $r_a>|Q_aP_a|>sr_c$. 
      But this is a contradiction. 
    So we have proven the conclusion of the theorem. 
\end{proof}

This theorem implies that if the dataset has several $s$-super-ball clusterings, then they are nested, or are refinements of each other.  

\begin{lemmaX}\label{lem:convexdist}
Under Kleinberg's consiostency transformation, not only the distances between {data}points from distinct clusters increase, but also between convex hulls of these clusters. 
\end{lemmaX}
\begin{proof}
    We concentrate on one cluster being subject of decreasing internal distances. We show that the distance of its elements to convex hulls of other clusters does not decrease.   Then we look at convex hulls of subsets of this cluster. 

Consider {data}points A,B from one cluster and the {data}point C from some other cluster, and a point D on the line segment AB, such that $|DA|/|DB|=p_D$.  From elementary geometry it is known that 


$$\cos (\angle ADC) \cdot |AD|\cdot |CD|= |AD|^2+|CD|^2-|CA|^2 $$
$$\cos (\angle BDC) \cdot |BD|\cdot |CD|= |BD|^2+|CD|^2-|BA|^2 $$
$$\cos (\angle ADC)= -\cos (\angle BDC)$$
Hence

\begin{equation*}
|AD|^2|AD|^{-1} |CD|^{-1}+|CD|^2|AD|^{-1} |CD|^{-1}-|CA|^2|AD|^{-1} |CD|^{-1}
\end{equation*}
$$+
|BD|^2|BD|^{-1} |CD|^{-1}+|CD|^2|BD|^{-1} |CD|^{-1}-|CB|^2|BD|^{-1} |CD|^{-1}
=0
$$
Multiplied with $|CD|$
%
\begin{equation*}
|AD|^2|AD|^{-1} +|CD|^2|AD|^{-1}-|CA|^2|AD|^{-1} 
+
|BD|^2|BD|^{-1}+|CD|^2|BD|^{-1} -|CB|^2|BD|^{-1} 
=0
\end{equation*}
Multiplied with $|AD|$
%
\begin{equation*}
|AD|^2| +|CD|^2-|CA|^2 
+
|BD|^2p_D+|CD|^2p_D -|CB|^2p_D 
=0
\end{equation*}
Rearranged
\begin{equation*}
|AD|^2 +|CD|^2 (1+p_D)
+
|BD|^2p_D-|CB|^2p_D 
-|CA|^2=0
\end{equation*}
In function of $|AB|$
\begin{equation*}
|AB|^2\frac{p_D^2}{(1+p_D)^2} +|CD|^2 (1+p_D)
+
|AB|^2\frac{p_D}{(1+p_D)^2}-|CB|^2p_D 
-|CA|^2=0
\end{equation*}
\begin{equation*}
|AB|^2\frac{p_D}{(1+p_D)} +|CD|^2 (1+p_D)
-|CB|^2p_D -|CA|^2=0
\end{equation*}
Note that upon consistency transformation, due to its definition, 
$|CA|$ and $|CB|$ can only be increased and 
$|AB|$ can only decreased. 
Therefore, given any of the above increases or decreases, $|CD|$ will grow. So upon Kleinberg's consistency transformation, not only the distance of a {data}point $C$ to endpoints of a line segment $AB$ grows, but also to any point of the 1-simplex $AB$. 
What about more complicated simplexes?
Given a $t$-simplex, $1<t\le d$, with corner data points $A_0,\dots,A_t$ in general positions in the same cluster,  and an internal point of the convex hull created by it $D=D_t$,
construct a line $A_tD_t$ and denote its intersection with the (t-1)-simplex $A_0,\dots,A_{t-1}$ as $D_{t-1}$, and 
define $p_{D_t}=|A_tD_t|/|D_tD_{t-1}$. 
Compute $p_{D_t}$ from the above formula setting $D_0=A_0$. 
So we get a sequence of points $D_1,\dots,D_t$ and of quotients $p_{D_1}.\dots,p_{D_t}$.
This should be done on the data after the Kleinberg's consistency operation. We can now construct the {counter}images of the points $D_i$ in the original simplex by selecting a $D_i$ such that  $p_{D_i}=|A_iD_i|/|D_iD_{i-1}$ for each $i=1,\dots,t$ with setting again $D_0=A_0$. 
Now, for each $C$ from a distinct cluster, in an iterative manner, starting with $i=1$, prove by the above result that with Kleinberg's consistency transformation, the distance $D_iC$ can only increase. 

With this knowledge let us consider the whole convex hull of {data}points of a cluster. Each point of the convex hull after Kleinberg's consistency transformation lies within the interior of a non-degenerate t-simplex with corners being {data}points in general positions. It may lie within more than one such simplexes. In this case, given any a-priori order of {data}points, we order the simplexes in lexicographic order and pick the first in this order. 
In such a way we get a unique counterpoint for each non-{data}point of the convex hull after the transformation and can prove based on its simplex that its distance to any {data}points from outside of this cluster is non-decreasing. 

The reasoning can be repeated now for any non-{data}points from convex hulls of other clusters, not only the {data}points, using the just defined mappings between convex hull points of clusters before and after the Kleinberg's transformations. 

And if the distances between any convex hull points of clusters do not decreases, so does not the distance between cluster convex hulls. 
\end{proof}

\begin{theorem}
If the dataset  $S$ in $\mathbb{R}^d$ has the property of 
$s$-super-ball-separation, $s>1$  into $k$ clusters,  
then, 
after applying Kleinberg's consistency transformation to the clustering consisting of these $k$ clusters,  the resulting clustering  
  has the property 
 of 
$s$-super-ball-separation  into $k$ clusters and possibly into $k_1> k$ clusters
\end{theorem}
\begin{proof}
The theorem is implied by Lemma \ref{lem:convexdist}. 
Kleinbergs consistency transform implies increase of distances between {data}points in different clusters. Hence the distance between convex hulls increases also. 

When the intra-cluster-distances are reduced,  the radius of the cluster cannot increase, and the cluster fits the former convex hull.  So,  the inter-cluster distances do not decrease so that altogether, $s$-super-ball-separation, $s>1$  into $k$ clusters is preserved. 
\end{proof}

There is an {in}consequence in Kleinberg's consistency transformation. By reducing the intra-cluster distances, the variance within the cluster will decrease. But increase of the inter-cluster distances combined with the intra-cluster decrease may result in decrease of the intra-luster variance too. Therefore let us strengthen Kleinberg's consistency by requering that the variance of any pair of clusters shall not decrease. It shall be called \emph{variance consistency transform}.
\begin{lemmaX}\label{lem:varcons}
    Variance consistency does not decrease the distances of cluster gravity centers. 
\end{lemmaX}
\begin{proof}
    The claim results from the fact that when the variances of clusters decrease due to consistency transformation reducing intra-cluster distances, then to keep the variance of a cluster pair, the distance between cluster centers must be increased. 
\end{proof}

Let us discuss the problem that Kleinerg's consistency transform may create new structures - $s$-{super}ball clusterings into $k$ clusters  with higher and/or lower $k$ than the original one, besides the one with the original $k$.  The extra clusters will appear within a previous cluster. 
This means that the so-called refinement consistency  will be kept by any algorithm discovering $s$-{super}ball separated clusters. This may not be a big discovery, as already Kleinberg pointed at a special single-link algorithm having the property of refinement consistency. We will show that a version of $k$-means has also this property which is important because single-link algoithm has quadratic complexity in the number of nodes, and this special version of $k$-means is linear in them.  

But can the creation of substructures by consistency transformation be prevented?
It has been proven in \cite[Section 3 on convergent consistency]{MAKRAK:2020:limcons} for the general case, that preventing consistency transformation from creating substructures   implies linear scaling of the distances within a cluster. 
Other types of preventing creation of new substructures can only be tailored either to concrete clustering algorithms or concrete realms of data structures.

Let us suggest here a method of substructure preventing consistency transformations for any cluster that cannot be split via any $s$-{super}ball-separation. 

\Bem{
\begin{method}\label{met:istree}
    Cover the cluster subject to consistency transformation with a minimum distance spanning tree. If for any {data}point in this tree, the quotient of the longest intersecting edge to the shortest one lies below $s$ after the consistency transformation,  then no new structures will emerge.  This is sufficient but not necessary. 
\end{method}
    
\begin{method}\label{met:construct}
Construct the image of the cluster subject to consistency transformation as follows: 
In $\mathbb{R}^d$ place the first $d+1$ {data}points (a simplex) in the space in general position so that the distances from one point to the all other are below the $s$ range. 
Each other {data}point of the cluster is added as follows: Create the convex hull of the already included points. Pick one of the corner points (point P) of the hull. Identify the closets point C of P in the current set. Then pick a point Q inside the convex hull and draw a  straight line PQ. On this line pick a point R so that P lies between Q and R so that the distance $|PR|<s\cdot |PC|$. This guarantees that there is no  $s$-{super}ball clustering into $k>1$ clusters.
    Afterwards {re}scale the distances between the points by the same factor so that the requirement of Kleinberg's consistency is matched. 
     This construction is sufficient but not necessary. 
\end{method} 
}
\begin{method}\label{met:moveinsimplex}
In the current cluster, 
take a {data}point P, its closest neighbour Q and the set Z of all {data}points not further from Q than $s\cdot PQ$ as well as the dataset $W$ of points not further than $s\cdot 2\cdot PQ$ from Q.  If $P$ lies within the convex hull of $Z$ then we can move $P$ within that convex hull increasing its distance from Q but paying attention that it will not get closer than to $Q$ to any point from $W$. In the same way proceed with other {data}points of the cluster. 
In the end  
 {re}scale the distances between the points by the same factor so that the requirement of Kleinberg's consistency is matched. 
     This construction is sufficient but not necessary. 
\end{method}

Let us show why
\begin{lemmaX}\label{lem:methmoveinsimplex}
    Method \ref{met:moveinsimplex} produces a cluster  that will not be split by a $s$-{super}ball-clustering discovery algorithm. 
\end{lemmaX}
\begin{proof}
    Upon $s$-{super}ball clustering any {data}point will belong to the same cluster as its closest neighbour. Denote the distance between P and Q prior to transformation with $d_1$. The distance $d_2$ between P and Q after transformation will be bounded by $d_2\le d_1\cdot s$. No point X further from Q than $2d_1s$ can be closer to P than $d_2$. So it is sufficient to check only {data}points from W for closeness to P and if none is closer to P than Q from this set, then the same applies to the entire cluster. 
    So due to our assumptions, P and Q are closest and hence belong to the same cluster prior and after the transformation.  
    Also the {data}points of the set Z will belong to the same set as Q because the distance to them is smaller than the distance to P times $s$.  But the distance of the convex hull of Z to the other cluster elements does not change upon moving P, so if there was no structure in the cluster prior to moving P, there will be none afterwards. 
\end{proof}

Note that the final step of consistency transformation after Method \ref{met:moveinsimplex}
would be:   {re}scale and move away clusters to keep variance consistency transformation requirements.

Having discussed the properties of $s$-{super}ball separated data, let us look for algorithm discovering the corresponding $s$-{super}ball clustering.


The Algorithm 1 in \cite{RAK:MAK:2019:perfectball}, here Algorithm \ref{alg:modincrkmeans}, originally designed to discover perfect-ball-clusterings, can discover the $s$-super-ball-clustering for any $s>1$ also.  

\begin{theorem}
If the data is  $s$-super-ball-separated for any $s>1$ into $k$ clusters, then the Algorithm \ref{alg:modincrkmeans} will return this clustering. 
\end{theorem}
\begin{proof}
After initializing the $k$ cluster gravity centers  in line 1, 
these centers may lie in distinct real clusters or may not. 
Line 4 initializes iteratively an excessive cluster gravity center. Therefore subsequently the number of clusters needs to be reduced back to $k$. If the real number of clusters is $k$ and we have $k+1$ cluster centers, then for sure two of them must lie in the same real cluster - line 5 identifies them as $t_i,t_j$ because they are the closest pair. 

The gravity center of any set of {data}points lies always within the convex hull of these points. 

Gravity center $t_i$ of any subset of the elements of cluster $i$ lies within $r_i$ from gravity center $t_j$ of any other (distinct) subset of elements of cluster $i$.  
By the very definition, the distance to gravity center $t_m$ of subset of elements  of cluster $m$ will amount to at least $sr_i+sr_m>r_m$.  
So lines 6-8 will in fact combine cluster centers from the same real cluster. 
\end{proof}

\begin{algorithm}
 \KwData{the data points $\mathbf{x}_i$, $i=1,\dots, m$, the required number of clusters $k$}
 \KwResult{T - the set of cluster centres  }
   Set $T = (t_1,\dots,t_k)$ to the first $k$ data points \; 
   Initialize the counts $n_1, n_2,\dots, n_k$ to 1 \; 
 \While{any data point {un}visited}{
  Acquire the next example, $t_{k+1}$. Set $n_{k+1}=1$ \;
 \If{ the distance between $t_i,t_j$, $j\ne i$, is lowest among all pairs from the set $(t_1,\dots,t_{k+1})$  }{
    Replace $t_i=(t_in_i+t_jn_j)/(n_i+n_j)$, thereafter $n_i=n_i+n_j$
\;
\If{$j\ne k+1$}{replace $t_j=t_{k+1}$, $n_j=n_{k+1}$\;}
   }
 }
 \caption{Sequential (incremental) $k$-means, our modification}
\label{alg:modincrkmeans}
\end{algorithm}

Algorithm 2 in \cite{RAK:MAK:2019:perfectball}
can be adapted to confirm/dis-confirm that the underlying clustering is a $s$-super-ball-clustering, see  Algorithm \ref{alg:modincrkmeanssecpass} here.

\begin{algorithm}
 \KwData{ $T = (t_1,\dots,t_k)$ be the resulting set of cluster centres from the Algorithm \ref{alg:modincrkmeans}.  
 \\$s$ - super-ball parameter
}
 \KwResult{{Cluster}ability decision }
Initialize the furthest neighbours   $f_1, f_2,\dots, f_k$ with  $t_1, t_2\dots,t_k$ respectively\; 
 \While{any data point {un}visited}{
Acquire the next example, $x$.  \;
 \If{$t_i$ is the closest centre to $x$ and $x$ is further away from $t_i$ than $f_i$}{
    Replace $f_i$ with $X$;
 }
}
Set $r_i$ as distance  between   $t_i$ and $f_i$.  \;
Compute distances $d_{ij}$ between each pair $t_i,t_j$ and 
compare it to $sr_i+sr_j>d_{ij}$ and $(2s+1)r_i+(2s+1)r_j<d_{ij}$
   \; 
\If{the latter is true for each pair}{We got an $s$-super ball clustering}{}
\If{the former is true for any pair}{ $s$-super ball   clustering was not found} {}
 \caption{Sequential $k$-means, our modification - second pass}
\label{alg:modincrkmeanssecpass}
\end{algorithm}

\begin{theorem}
Given a clustering result from Algorithm \ref{alg:modincrkmeans}, the incremental Algorithm \ref{alg:modincrkmeanssecpass} may confirm for $s>1$ that the result is a $s$-super-ball-clustering or dis-confirm that it is a $s$-super-ball-clustering, is correct if any of two decisions is made.   
\end{theorem}
\begin{proof}
$r_i$ as computed in the algorithm is in fact the lower estimate of the intrinsic radius used in $s$-ball-separation definition, whereas $2r_i$ is its upper estimate.
If there is a $s$-{super}ball-separation of the clusters, then the distance between cluster gravity centers may lies somewhere between $s(r_i+r_j)$ and $(2s+1)(r_i+r_j)$. Hence the respective decisions at the end of the algorithm. 
\end{proof}

We see that the confirmation /{dis}confirmation of $s$-{super}ball clustering is not simple in a single pass and there is a grey zone in-between. This gap originates probably from the imprecise conceptualization of the idea of clustering, or more precisely how to conceptualize the boundary of clusters.

Let us demonstrate that Kleinberg's axioms can be nearly matched in the realm of super-ball-clustering. 
Let us introduce the  concept of half-richness. 
Kleinberg's richness does not fit quite into the realm of sup[er-ball-separation because each "{super}ball" needs to contain at least two {data}points. Therefore the 
\begin{definition}
\emph{Half-richness} is the richness restricted to those clusterings that have at least two {data}points in a cluster.
\end{definition}
Obviously, even half-richness replacing the richness does not make original Kleinberg's axioms free from contradictions. 
Let us also modify consistency axiom so that no internal structure is introduced. 
Let us discuss only the method-\ref{met:moveinsimplex}-consistency 
which should be granted via variance consistency transformation  with method-\ref{met:moveinsimplex}.

We need a special clustering algorithm.
$k$-means, with $k$ ranging over a set of values, $[k_0,k_1$, if we assume that it returns the    $s$-super-ball-separated  $k$-clusterings for the largest possible $k$ from the above range (excluding too small clusters), we call it \emph{max-$k$-$[k_0,k_1$]-$s$-means algorithm}. It is an extension of Algorithm \ref{alg:modincrkmeans} in that we check one by one all $k$ from the range $[k_0,k_1$] and with Algorithm \ref{alg:modincrkmeanssecpass} we decide for which $k$ the $s$-{super}ball clustering was found, and we choose the largest $k$ with this property (that is discovery of Algorithm \ref{alg:modincrkmeans} confirmed by Algorithm\ref{alg:modincrkmeanssecpass}). 

\begin{theorem} 
Half-richness, method-\ref{met:moveinsimplex}-variance-consistency and scale-invariance are non-contradictory properties in the super-ball-separations realm of $s$-super-ball-clusterings with $s>1$ (that is for algorithms discovering super-ball-clusterings). 
\end{theorem}
\begin{proof}
To prove non-contradiction, we need at least one algorithm that matches these requirements. 
We investigate here the just mentioned  max-$k$-$[k_0,k_1$]-$s$-means algorithm.
Let us consider only {data}sets where no two {data}points lie at the same place in space (distances greater than zero).
Let us demonstrate first the half-richness. 

If we know the true number of clusters $k=k_t$ in advance, the Algorithm \ref{alg:modincrkmeans}  (incremental $k$-means) will always discover a $s$-super-ball-clustering $\Gamma_s$ into $k_t$ clusters.  If we choose a too large $k>k_t$, then Algorithm \ref{alg:modincrkmeans}  will split some of the clusters into {sub}clusters, however, never creating a cluster as a mix of subsets of $\Gamma_s$  clusters. Hence an iterative procedure starting from a large $k$ ($\le n$ - the number of cluster elements) and decreasing at each step the number $k$ will finally detect the clustering $\Gamma_s$ if it exists. Some additional conditions are necessary, like non-zero radii $r_i$. 

If we want to work within the realm of super-ball-clusterings, then obviously Kleinberg's richness cannot be valid for any algorithm. Only half-richness  property  (all possible clusterings of $k=1,\dots,n/2$ clusters) is realistic. 
Half-richness has no conflict with the scale-invariance property, but it is in conflict with Kleinberg's consistency.

Therefore, the modified consistency is used.  
\Bem{
For any cluster element $e$ let $E$ be a sequence of same cluster elements such that $d(e,E_i)\le d(e,E_{i+1})$. So let us impose the restriction that for each $i$ $5d(e,E_i)>   d(e,E_{i+1})$ after consistency transform and let call this transform \emph{5-consistency}.
In this way, no super-ball-separation is created inside a cluster.
}
If the algorithm has detected a $s$-{super}ball clustering into $k$ clusters for some $k$, then it is based essentially on the distances between cluster centers. The modified consistency preserves the distances between clusters (Lemma \ref{lem:convexdist}) and also between the cluster centers (Lemma \ref{lem:varcons}). Hence applied after the consistency transformation, clustering into $k$-clusters will be found. No higher number of clusters will be detected, because further clusters would have to be contained in the existent ones, but in the existent ones no new $s$-{super}ball separated structures are created (Lemma \ref{lem:methmoveinsimplex}).  

The scaling invariance is obvious. This completes the proof. 
\end{proof}

\section{Final Remarks}\label{sdec:conclusion}

In this paper we have shown that the richness axiom proposed by Kleinberg for clustering is unacceptable in the light of learnability theory and due to its conflicts with Kleinberg's consistency axiom. 

We have shown that the problems with this axiom need a resolution not only in terms of restricting the number of clusterings to be considered, but also a change in consistency axiom is needed. We proposed to either use the consistency transformation or by going over to the world of super-ball-separation. 

Note that the centric consistency may seem to be quite rigid, but still it turns out to be less restrictive than Kleinberg's consistency as distances between {data}points in different clusters are allowed to get closer.

The super-ball-clustering is  very rare case for real world applications and its usefulness is first of all of theoretical nature, that is to show that a sound  axiomatic system can be created, contrary to general opinion in literature. 
Though super-ball-separation is clearly too idealistic, but in practice we encounter situations where in fact the  {data}points outside of a strict cluster core are sparse so that an extension of this concept to an approximated one seems to be an interesting area of further research.
The sound axiomatic crisp case, investigated here, may be a good starting point for creating sound approximate definitions for clustering function and clustering axioms.

\bibliographystyle{splncs04}

\bibliography{MeineBibliographie_bib, V5_centricconsistencyRAKMAK_bib}

\end{document}

%% file: varTable1.tex
\begin{table}
\centering
\caption{Variance explained (in percent)
when applying $k$-means algorithm with 
$k=2,\dots,10$ to data from Figures 
\ref{fig:8clusters} (a) (Original),  
\ref{fig:8clustersKleinberg} (Kleinberg),
   and 
\ref{fig:8clustersKlopotek}(b) (Centric)  }
\label{tab:comparisonVariance}
\begin{tabular}{|l|ll|ll|ll|}
\hline 
$k$ & Original &RIV & Kleinberg &RIV& Centric &RIV\\
\hline  
2  &  49.5  & 2.36 & 78.2  &\textbf{7.89}  &50.1 &2.36\\ 
3  &  70.7  & 2.78&  88.1  &1.25  &71.3 &2.78\\ 
4  &  78.1  & 1.43&  96    &4.38  & 78.9 &1.43\\ 
5  &  83.3  & 1.20&  97.8  &1.63  & 84.2 &1.20\\ 
6  &  87.6  & 1.10&  98.9  &2.75  & 88.6 &1.10\\ 
7  &  91.6  & 1.60&  99.3  &2.00  & 92.6 &1.60\\ 
8  &  94    & \textbf{8.33}&  99.5  &2.00  & 95.1 &\textbf{8.33}\\ 
9  &  94.3  & 1.50&  99.6  & &  95.4 &1.50\\ 
10  &  94.6  & &  99.6     & &  95.6 &\\ 
\hline 
\end{tabular}
\end{table}